\documentclass[twoside,11pt]{article}

\usepackage{jmlr2e}

\usepackage[T1]{fontenc}
\usepackage[utf8]{inputenc}
\usepackage[margin=1in]{geometry}
\usepackage{amsmath,amssymb,amsthm,mathtools,bm}
\usepackage{graphicx}
\usepackage{subcaption}
\usepackage{booktabs,tabularx}
\usepackage{enumitem}
\usepackage{microtype}
\usepackage{comment}
\usepackage{xcolor}
\usepackage{mathtools}
\graphicspath{{../Figures/}}
\setlist{nosep}
\theoremstyle{plain}
\newtheorem{theorem}{Theorem}[section]
\newtheorem{proposition}[theorem]{Proposition}
\newtheorem{lemma}[theorem]{Lemma}
\newtheorem{corollary}[theorem]{Corollary}
\theoremstyle{definition}
\newtheorem{assumption}[theorem]{Assumption}
\newtheorem{definition}[theorem]{Definition}
\newtheorem{example}[theorem]{Example}
\theoremstyle{remark}
\newtheorem{remark}[theorem]{Remark}
\numberwithin{equation}{section}

\newcommand{\R}{\mathbb{R}}
\newcommand{\E}{\mathbb{E}}
\newcommand{\Pp}{\mathbb{P}}
\newcommand{\B}{\mathcal{B}}
\newcommand{\norm}[1]{\left\lVert #1\right\rVert}
\newcommand{\diag}{\operatorname{diag}}
\newcommand{\softmax}{\operatorname{softmax}}
\newcommand{\sigm}{\operatorname{sigm}}
\newcommand{\ac}{a_{\mathrm{c}}}
\DeclareMathOperator{\eig}{eig}

\usepackage{lastpage}
\firstpageno{1}

\hypersetup{hidelinks}

\begin{document}

\title{Dynamics to decision: A mathematical theory of Lyapunov spectra and decision boundaries in deep classifiers}

\author{\name Shirin Panahi \email s.panahi@colostate.edu\\
       \addr Department of Electrical and Computer Engineering\\
       Colorado State University\\
       Fort Collins, Colorado, USA
       \AND
       \name Amirhossein Nazerian \email a.nazerian@colostate.edu \\
       \addr Department of Mechanical Engineering\\
       Colorado State University\\
       Fort Collins, Colorado, USA
       \AND
       \name Ali Pezeshki \email ali.pezeshki@colostate.edu\\
       \addr Department of Electrical and Computer Engineering\\
       Colorado State University\\
       Fort Collins, Colorado, USA}

\editor{}
\maketitle

\begin{abstract}
A deep classifier is defined not only by the decision it produces, but also by the sequence of transformations through which that decision is formed. Treating this evolution as a dynamical system across layers provides a natural framework for asking how decision geometry emerges through depth and how far back we can trace a boundary's dynamical signature. We model a feed-forward classifier as a finite, nonautonomous discrete dynamical system, with layers playing the role of discrete time steps.
We study the Finite-Time Maximum Lyapunov Exponent (FTMLE) of the data samples' dynamical trajectory through depths of the classifier. The FTMLE measures the rate of convergence/divergence of nearby trajectories.
We move the observation endpoint backward from probabilities to logits and then to hidden representations.
For Gaussian classes, we prove that probability-level FTMLE carries a clear geometric signature of the decision boundary, with its dominant direction aligned with the boundary normal. Moving one step backward to the logits, we prove this relationship is no longer universal but depends critically on how the classifier is trained, particularly on the choice of loss function. Moving further backward to the hidden representation, the connection becomes more conditional: boundary-related FTMLE can persist, but only under identifiable structural conditions. 
We propose geometry-aware fine-tuning for restructuring the classifier's hidden FTMLE, and propose conditions for guaranteed concentration of high hidden FTMLE near the decision boundary.
Through our numerical results, we show the generality and validity of our theoretical results.
Understanding the evolution of data samples as traveling through the layers of classifier provides a principled foundation for identifying where boundary-relevant sensitivity emerges and for developing layer-aware regularization strategies.
\end{abstract}


\section{Introduction}\label{sec:introduction}

Two classifiers can draw the same decision boundary while deforming nearby inputs in very different ways before producing a label. The boundary specifies where class labels change, but it does not reveal which perturbation directions expand, whether that expansion contributes to the decision, or at what depth it arises. Understanding these distinctions matters since it reveals how decision geometry is formed through depth, rather than inferred from predictive performance alone.
In particular, a region of extreme deformations is not automatically a region where the predicted class changes. 
Here, we aim to identify under what conditions these extreme deformations are properly utilized in the decision formation. 

A natural way to pose the question comes from nonlinear dynamics. A feed-forward classifier can be written as \(h_{k+1}=F_k(h_k)\): layer index plays the role of discrete time, and the learned map \(F_k\) can differ at every step. 
The Lyapunov exponent analysis measures the exponential rate of convergence/divergence of nearby solutions of this discrete-time dynamical system. 
A positive Finite-Time Maximum Lyapunov Exponent (FTMLE) indicates divergence of nearby trajectories, while a negative FTMLE suggests convergence in a finite window of time~\citep{Ott_2002,strogatz2024nonlinear}. The observation endpoint is part of this measurement. We begin at the probability output and then move the endpoint backward in time: first to the logit, then to the hidden representation. Each endpoint retains different information about class separation and local sensitivity.

Recent studies have provided empirical evidence that FTMLE can reflect the geometry of neural-network decision boundaries. Across deep classifiers, regions of elevated FTMLEs have been observed near class-separating boundaries, suggesting that the network's local expansion field can carry information about where classification decisions change~\citep{Kondo2021,Storm2024,Kuehn2026}. These findings establish an empirical connection between finite-time expansion and decision geometry, while leaving open a more fundamental question: under what conditions, and at which stages of the network evolution, does an expansion ridge actually identify a decision boundary?

Complementary lines of machine-learning research explain why
Jacobian geometry deserves attention. Analyses of singular-value
propagation relate network depth and initialization to trainability
\citep{Pennington2017,Pennington2018, Poole2016, Schoenholz2017, Haber2018}. For trained networks,
input-output Jacobian norms has been examined in relation to
generalization \citep{Novak2018}, while Jacobian-based margin bounds and spectral complexity
connect local sensitivity to classification margins
\citep{Sokolic2017, Bartlett2017, Neyshabur2018, Tsuzuku2018, Yoshida2017}. Instance-specific
adversarial robustness bounds provide a complementary perspective
\citep{Hein2017, Weng2018, Han2024, Fawzi2016Random}. The geometry of adversarial vulnerability has also been studied through local gradients, distances to class-changing perturbations, curvature of decision boundaries, and singular directions of classifier Jacobians~\citep{Goodfellow2015,MoosaviDezfooli2016,Papernot2016,Fawzi2018,MoosaviDezfooli2018Universal,Khrulkov2018,Paniagua2025}. In parallel, a substantial body of work has sought to control these differential properties directly during training with early tangent- and contraction-based penalties and regularization~\citep{Simard1991TangentProp,Rifai2011,Jakubovitz2018,Ross2018,Hoffman2019,Finlay2020,Finlay2021,Wu2024,Cisse2017,Anil2019,Johansson2022,Meunier2022}. 
Collectively, this literature establishes that the local differential response of a classifier contains information relevant to generalization, margins, and vulnerability, but these quantities are typically used as sensitivity measures, bounds, or regularization objectives rather than as spatial fields intended to localize a decision boundary.

A complementary line of work has examined the geometry of learned representations and decision regions more directly. Piecewise-linear and spline-based methods characterize and visualize neural-network decision geometry~\citep{Balestriero2018,Humayun2023SplineCam}. Representation-level Jacobian and information-geometric analyses show that learning reshapes local geometry near decision boundaries and links differential sensitivity to inter-class separation and open-set behavior~\citep{ZavatoneVeth2025,Park2024}. Spectral and gradient-based control of intermediate representations has been explored to improve adversarial robustness, out-of-distribution behavior, and distance-aware prediction~\citep{Nassar2020,Yang2025Spectral,Sharifi2024,Liu2020SNGP}. These results strengthen the broader view that classification is accompanied by a nontrivial reshaping of local geometry throughout the network. They also suggest that the geometry observed at an intermediate representation is not necessarily equivalent to that observed at the final classifier output.

Information-geometric approaches make the link between sensitivity and decision geometry more explicit by examining metrics induced by the predictive distribution. Fisher-based methods have been used to characterize locally sensitive directions, adversarial vulnerability, robustness, and task-relevant geometry through the spectrum and eigen-directions of the induced metric~\citep{Miyato2018,Zhao2019,Shen2019Fisher,Picot2023,Tron2024,ShiGarrier2024,Zhang2026Fisher,Sengupta2026Fisher}. Related work further shows that Fisher-induced representation geometry can exhibit enhanced expansion near category boundaries and encode discriminative directions~\citep{BonnasseGahot2025}. Together with empirical observations of large finite-time Lyapunov expansion near class-separating regions~\citep{Kondo2021,Storm2024,Kuehn2026}, these studies further support the idea that differential properties of a learned mapping can encode information about classification geometry. However, the Jacobian under consideration, the network endpoint at which it is evaluated, and the specific notion of sensitivity vary substantially across these formulations. Consequently, although these contributions address important aspects of network sensitivity and boundary geometry, they do not establish when the relationship between a finite-time expansion field and decision geometry is theoretically well founded. In particular, it remains unclear when such a field can reliably be used to assess boundary-relevant sensitivity, how this relationship depends on where the expansion is measured within the network, and how the boundary signature changes as the observation point is moved backward through depth.

We answer this question through a sequence of results that follows
the observation endpoint backward through the classifier: at the probability endpoint, we prove that the population-optimal posterior for two equal-prior Gaussian classes with common spherical covariance has an exact normal expansion profile whose maximum set is the Bayes boundary. 
For regular-simplex multiclass Gaussian models, we derive the full probability-Jacobian singular spectrum and show that the leading input direction is normal on every smooth active pairwise facet. Magnitude localization is more delicate: an explicit posterior-dependent coefficient determines whether a facet is a normal maximum or minimum, with a sharp transition in the three-class case. Thus normal alignment alone does not imply a boundary ridge.
To the best of our knowledge, this paper is the first to report that the decision boundary may be a local normal minimum of the FTMLE (also the sensitivity).

At the binary logit endpoint, we show that population mean-squared error on a signed target produces a boundary-maximal expansion profile, whereas cross-entropy produces an affine Bayes logit with spatially constant expansion in the same Gaussian model. The two logits have the same sign boundary. The contrast shows why a probability-level ridge cannot simply be transferred to the preceding logit.
Therefore, we prove that the expansion does not always have a maximum on the decision boundary.

At the representation endpoint, we separate full hidden expansion from expansion visible to the current readout. We then prove sufficient, uniform conditions under which a task-visible normal decrease transfers to a full hidden-state ridge.
We complement this theory with a matched-continuation two-moons study. It illustrates substantially stronger localization of hidden expansion around the learned boundary after geometry-aware fine-tuning, while keeping classification accuracy nearly unchanged.

The rest of the paper is organized as follows:
Section~\ref{sec:framework} establishes the dynamical and
statistical setting. Sections~\ref{sec:probability}-\ref{sec:representation}
analyze the probability, logit, and representation endpoints in
that order; Section~\ref{sec:conclusions} provides conclusions. Complete proofs appear in the appendices.

\section{Mathematical framework}\label{sec:framework}

\textbf{Deep classifiers as finite nonautonomous discrete dynamical
systems.}
Let \(\Omega\subseteq\R^d\) be the input domain, let \(K\geq1\) be the
number of hidden propagation steps, set \(d_0=d\), and let
\(F_k:\R^{d_k}\to\R^{d_{k+1}}\). A depth-indexed classifier is written as
\begin{equation}\label{eq:nonautonomous-system}
  h_0(x)=x,\qquad
  h_{k+1}(x)=F_k\bigl(h_k(x)\bigr),
  \qquad k=0,\ldots,K+2,
\end{equation}
where \(h_k(x)\in\R^{d_k}\) is the state (propagated data sample $x$) through the layers of the deep classifier, up to layer $k$. 
The maps vary with \(k\), making this a finite
nonautonomous system. The three principal observation endpoints are
\begin{equation*}
  h_K(x)=h_\theta(x),\qquad
  h_{K+1}(x)=\ell_\theta(x),\qquad
  h_{K+2}(x)=p_\theta(x),
\end{equation*}
representing the final hidden state (hidden representation), logit, and probability layer, respectively.
Here \(F_{K+1}\) is the readout and \(F_{K+2}\) is the probability link.

\textbf{Finite-time Lyapunov spectra.}
During inference, each input $x$ follows a finite trajectory through a sequence of learned nonlinear transformations before a final readout assigns a class. Understanding how nearby trajectories deform through this evolution provides information about how the classifier constructs its decision, rather than only observing its final prediction.
Lyapunov exponent analysis is a classical tool for characterizing such trajectory deformations \citep{Ott_2002,strogatz2024nonlinear, khalil2002}. 
In the following, we drop the explicit notation of $x$ dependence from $h_k (x)$ for simplicity.
Consider a trajectory initialized at \(h_0\) and an infinitesimal perturbation \(\delta h_0\). After \(k\) steps, $h_0 \to h_k$, and $\delta h_0 \to \delta h_k$.
To first order, 
\begin{align*}
    \delta h_k & = \Big[ D(F_{k-1}\circ\cdots\circ F_0)(h_0) \Big]\delta h_0  \eqqcolon \left[ D F_0^{(k)}(h_0) \delta \right] h_0.
\end{align*}
The Lyapunov exponent in the initial direction
\(u_0={\delta h_0}/{\|\delta h_0\|}\) is defined as 
\[\lambda(h_0,u_0) 
\coloneqq 
\lim_{k\rightarrow\infty}
\frac{1}{k}
\log
\left(\frac{\|\delta h_k\|}{\|\delta h_0\|}\right).
\] 
It measures the average exponential rate at which nearby trajectories diverge or converge.
Because a classifier has finite depth, the corresponding finite-time quantity is the natural object of interest. 
Maximizing over $u_0$ gives the Finite-Time Maximum Lyapunov Exponent (FTMLE),
\[\lambda_{F_0}^{(k)}(h_0)= \max_{\|u\|=1}
\frac{1}{k}
\log
\left\|
    DF_0^{(k)}(h_0)u
\right\|
=
\frac{1}{k}
\log
\sigma_1\!\left(DF_0^{(k)}(h_0)\right),
\]
where \(\sigma_1(\cdot)\) is the largest singular value. A positive (negative) FTMLE indicates local expansion (contraction).
More generally, the finite-time Lyapunov spectrum is
    $\lambda_{i,f_0}^{(k)}(h_0)
    =
    \frac{1}{k}
    \log
    \sigma_i\!\left(Df_0^{(k)}(h_0)\right)$,
where \(\sigma_i(\cdot)\) is the \(i\)th singular value.
We will also denote the \(i\)th right singular vector via $v_i (\cdot)$.

\textbf{FTMLE in deep classifier endpoints.}
For a deep classifier, we use input data sample $x$ as the initial state, i.e., $h_0 = x$.
At the three endpoints (introduced earlier), the following simplified notation is used throughout:
\begin{align*}
\lambda_H (x) & =\frac{1}{K}\log\Lambda_H (x),
   & \text{where} \quad \lambda_H 
  & = \lambda_{F_0}^{(K)}(x), \quad && \Lambda_H (x)   = \frac{1}{K} \log \sigma_1\!\left(DF_0^{(K)}(x)\right),
  \\
  \lambda_\ell (x) & =\frac{1}{K+1}\log\Lambda_\ell (x),
   & \text{where} \quad  \lambda_H 
  & = \lambda_{F_0}^{(K+1)}(x), \quad && \Lambda_\ell (x)  = \frac{1}{K+1} \log \sigma_1\!\left(DF_0^{(K+1)}(x)\right),
  \\
 \lambda_p (x) & =\frac{1}{K+2}\log\Lambda_p (x),
  \quad & \text{where} \quad  \lambda_H 
  & = \lambda_{F_0}^{(K+2)}(x), \quad && \Lambda_p (x)  = \frac{1}{K+2} \log \sigma_1\!\left(DF_0^{(K+2)}(x)\right).
\end{align*}
We use \(\Lambda\) for raw expansion and \(\lambda\) for the FTMLE. 

\begin{definition}[Normal FTMLE ridge]\label{def:normal-ridge}
Let \(x_0\) lie on a smooth boundary facet with unit normal \(n(x_0)\).
For a specified endpoint label \(e\in\{H,\ell,p\}\), the boundary has a strict
normal FTMLE maximum at \(x_0\) if \(t=0\) is a strict local maximum of
\(\lambda_e(x_0+t\,n(x_0))\).  A collection of such points is a normal FTMLE
ridge for endpoint \(e\).
\end{definition}
We ask two separate questions. \emph{Directional identification} means
that a leading right singular vector of \(J_k\) is normal to the decision
boundary. \emph{Magnitude localization} means that the leading expansion
has a normal maximum on that boundary. A normal direction alone does not
settle the magnitude question.


\subsection{Statistical, probabilistic, and geometric
setting}\label{sec:statistical-setting}

Write \(\eta_i(x)\coloneqq \Pp(Y=i\mid X=x)\) for the Bayes posterior; for binary
classification, \(\eta(x)\coloneqq \eta_1(x)\) is its positive-class component.
For a binary network, \(p_\theta:\Omega\to[0,1]\) is the predicted
positive-class probability
\begin{equation*}
  p_\theta(x)=\sigm\bigl(\ell_\theta(x)\bigr),
  \qquad
  \sigm(s)=(1+e^{-s})^{-1},
\end{equation*}
and
\(\B_\theta=\{x:p_\theta(x)=1/2\}=\{x:\ell_\theta(x)=0\}\).
The binary Bayes boundary is
\(\B^\star\coloneqq \{x:\eta(x)=1/2\}\). For \(C\geq3\), \(p_\theta\) and
\(\ell_\theta\) are vectors and
\begin{equation*}
  p_{\theta,i}(x)
  =\frac{\exp(\ell_{\theta,i}(x))}
  {\sum_{m=0}^{C-1}\exp(\ell_{\theta,m}(x))},
  \qquad i=0,\ldots,C-1.
\end{equation*}
The active learned boundary and the smooth active Bayes facet are denoted by
\begin{align*}
  \B_{\theta,ij}
  &\coloneqq \{x:p_{\theta,i}(x)=p_{\theta,j}(x)
  =\max_m p_{\theta,m}(x)\},\\
  \mathring{\B}_{ij}^{\star}
  &\coloneqq \{x:\eta_i(x)=\eta_j(x)=\max_m\eta_m(x)
  \text{ with no higher-order tie}\}.
\end{align*}
The full learned multiclass boundary is
\(\B_\theta\coloneqq \bigcup_{i<j}\B_{\theta,ij}\), with an analogous
\(\B^\star\) for the Bayes posterior. A pairwise equality that is not
maximal among classes is not an active decision boundary. Individual
multiclass softmax logits have a common-function gauge; their pairwise
differences, and hence the decision regions, are identified by \(p_\theta\).

\begin{assumption}[Differentiability and population
well-specification]\label{ass:learning}
The endpoint under study is continuously differentiable on its stated
domain. Its model class contains the relevant population optimizer, and
the stated population risk attains its global minimum in that class.
Whenever a derivative of the learned optimizer is identified with a
Bayes-posterior derivative, the input density is positive on the domain
and both maps are continuously differentiable there.
\end{assumption}

\begin{lemma}[Recovery by proper losses]\label{lem:proper-loss}
At each fixed input, binary cross-entropy and squared probability loss
have the unique minimizing prediction \(p=\eta(x)\). Categorical
cross-entropy and Brier loss have the unique minimizing probability
vector \(p=\eta(x)\) on the multiclass simplex. Under
Assumption~\ref{ass:learning}, any population-optimal probability map
equals the Bayes posterior pointwise on the domain where the positive
density and continuity conditions hold.
\end{lemma}

The proof is in Appendix~\ref{app:statistical-details}.
Its pointwise conclusion matters: an almost-everywhere identity alone
cannot be differentiated without a regularity argument. In the Gaussian
models below, the mixture density is strictly positive on \(\R^d\), and
continuity supplies that argument.

For \(p=\softmax(\ell)\), define
\begin{equation*}
  G(p)\coloneqq \diag(p)-pp^\top,
  \qquad D_\ell\softmax(\ell)=G(p).
\end{equation*}
This matrix is positive semidefinite, satisfies \(G(p)\mathbf 1=0\),
and records how the probability link reshapes the input-to-logit
Jacobian. At the probability endpoint
\(Dp_\theta=G(p_\theta)D\ell_\theta\) for multiclass softmax.

\begin{assumption}[Equal-prior spherical Gaussian
classes]\label{ass:gaussian}
Let \(C\geq2\), \(\sigma>0\), and let
\(\mu_0,\ldots,\mu_{C-1}\in\R^d\) be distinct. Assume
\begin{equation*}
  \Pp(Y=i)=\frac{1}{C},
  \qquad
  X\mid Y=i\sim\mathcal{N}(\mu_i,\sigma^2I_d).
\end{equation*}
\end{assumption}

Equal priors and common covariance give explicit posterior and boundary
geometry. They are used for global formulas in
Sections~\ref{sec:probability} and~\ref{sec:logit}; the local
signed-distance and representation results have their own conditions.

\begin{assumption}[\(C\)-class regular-simplex Gaussian
model]\label{ass:cclass}
Under Assumption~\ref{ass:gaussian}, let \(C\geq3\), \(d\geq C-1\),
\(c=C^{-1}\sum_i\mu_i\), and \(\nu_i=\mu_i-c\). Assume
\begin{equation*}
  \sum_{i=0}^{C-1}\nu_i=0,\qquad
  \norm{\nu_i}_2^2=\rho^2,\qquad
  \nu_i^\top\nu_j=-\frac{\rho^2}{C-1}\quad(i\neq j).
\end{equation*}
The common pairwise distance obeys
\(d_\mu^2=2C\rho^2/(C-1)\). Define the unit pairwise normal
\(n_{ij}\coloneqq (\nu_i-\nu_j)/d_\mu\) and signed coordinate
\(r_{ij}(x)\coloneqq n_{ij}^\top(x-c)\).
\end{assumption}

\begin{remark}[Scope of the structured models]\label{rem:model-scope}
The equal-prior Gaussian and regular-simplex assumptions give
closed-form population boundary geometry. The local signed-distance
result in Section~\ref{sec:binary-probability} and the representation
theorem in Section~\ref{sec:representation} state different, explicitly
local conditions. If Gaussian priors are unequal, log-prior terms shift
the discriminants; none of the equal-prior formulas below silently
assumes that extension.
\end{remark}

\section{Input-to-probability FTMLE}\label{sec:probability}

\subsection{Binary probability maps}\label{sec:binary-probability}

Let \(C=2\), and define
\begin{equation}\label{eq:binary-geometry}
  c\coloneqq \frac{\mu_0+\mu_1}{2},\qquad
  \delta\coloneqq \mu_1-\mu_0,\qquad
  n\coloneqq \frac{\delta}{\norm{\delta}_2},\qquad
  r(x)\coloneqq n^\top(x-c),\qquad
  \alpha\coloneqq \frac{\norm{\delta}_2}{\sigma^2}.
\end{equation}
Here \(r(x)\) is signed Euclidean distance from the perpendicular
bisector, and \(n\) points toward the mean of class \(1\).

\begin{theorem}[Binary input-to-probability Gaussian classes]\label{thm:binary-gaussian}
Under Assumptions~\ref{ass:learning} and~\ref{ass:gaussian}, with
\(\Omega=\R^d\) and notation~\eqref{eq:binary-geometry}, the
population-optimal probability map and Bayes boundary satisfy
\[
  p_{\theta^\star}(x)=\eta(x)=\sigm(\alpha r(x)),
  \qquad \B^\star=\{x:r(x)=0\}.
\]
The probability-endpoint expansion and FTMLE are
\begin{equation}\label{eq:binary-probability-profile}
  \Lambda_p(x)
  =\frac{\alpha}{4}
   \operatorname{sech}^2\!\left(\frac{\alpha r(x)}{2}\right),
  \qquad
  \lambda_p(x)=\frac{1}{K+2}\log\Lambda_p(x).
\end{equation}
Consequently,
\(\operatorname*{arg\,max}_{x\in\R^d}\lambda_p(x)=\B^\star\).
Moreover \(Dp_{\theta^\star}(x)\) has rank one, and its unique
nonzero right singular direction is \(\pm n\).
\end{theorem}

\noindent
The theorem establishes both magnitude localization and normal-direction
selection. Its profile has a maximum of \(\alpha/4\) at \(r=0\), decays
symmetrically away from the boundary, and depends on class separation and
noise only through \(\alpha\). 

\begin{proposition}[Local signed-distance posterior
criterion]\label{prop:binary-probability}
Let \(U\subseteq\Omega\) be a tubular neighborhood of a smooth binary
Bayes boundary. Suppose \(r\in C^1(U)\) is its signed-distance function,
\(\norm{\nabla r(x)}_2=1\) on \(U\), and
\(\eta(x)=\sigm(\alpha r(x))\) there for some \(\alpha>0\).
Under posterior recovery on \(U\), Equation~\eqref{eq:binary-probability-profile}
holds throughout \(U\), with \(r\) in place of the affine coordinate.
Every boundary point in \(U\) is a strict normal FTMLE maximum, and the
leading input singular direction is \(\pm\nabla r(x)\).
If the signed-distance, posterior-profile, differentiability, and recovery
conditions hold throughout \(\Omega\), and \(\B^\star\cap\Omega\neq
\varnothing\), then
\[
  \operatorname*{arg\,max}_{x\in\Omega}\lambda_p(x)
  =\B^\star\cap\Omega.
\]
\end{proposition}

The signed-distance condition makes the statement local to an arbitrary
smooth boundary and makes clear what the Gaussian model supplies: an exact
posterior profile whose derivative peaks at zero distance. 
Complete proofs of all results in Sec.~\ref{sec:binary-probability} are in
Appendix~\ref{app:binary-probability}.

\subsection{Regular-simplex multiclass probability
maps}\label{sec:multiclass-probability}

In the binary result, the output derivative has one nonzero singular
value. For \(C\geq3\), class competition makes the entire posterior
vector relevant. Regular-simplex means keep pairwise boundary geometry
explicit while retaining this spectral interaction. Let
\(V\in\R^{C\times d}\) have rows \(\nu_i^\top\).
Let \(\eig_j(G)\) denote the eigenvalues of \(G\) in nonincreasing order.
\begin{proposition}[Posterior, boundary, and spectral
reduction]\label{prop:cclass-core}
Under Assumptions~\ref{ass:learning}, \ref{ass:gaussian}, and
\ref{ass:cclass}, with \(\Omega=\R^d\),
\begin{equation}\label{eq:simplex-posterior}
  p_{\theta^\star,i}(x)=\eta_i(x)
  =\frac{\exp(\nu_i^\top(x-c)/\sigma^2)}
  {\sum_{m=0}^{C-1}\exp(\nu_m^\top(x-c)/\sigma^2)}.
\end{equation}
The identifiable pairwise logit differences obey
\[
  \ell_{\theta^\star,i}(x)-\ell_{\theta^\star,j}(x)
  =\log\frac{\eta_i(x)}{\eta_j(x)}
  =\frac{d_\mu}{\sigma^2}r_{ij}(x),
\]
and the active Bayes boundaries are the active Euclidean Voronoi facets
of the means. 
Further,
\begin{align}
  \sigma_j\!\left(Dp_{\theta^\star}(x)\right)
  &=\frac{d_\mu}{\sqrt{2}\sigma^2}
   \eig_j\!\left(G(\eta(x))\right),
  \qquad j=1,\ldots,C-1,
  \label{eq:multiclass-full-spectrum}
\end{align}
where the remaining singular values are zero. 
The expansion and FTMLE are, respectively,
\begin{align}
  \Lambda_p(x)
  &=\frac{d_\mu}{\sqrt{2}\sigma^2}
   \eig_{\max}\bigl(G(\eta(x))\bigr),
  \qquad
  \lambda_p(x)=\frac{1}{K+2}\log\Lambda_p(x).
  \label{eq:multiclass-spectral-reduction}
\end{align}
\end{proposition}
\begin{remark}[Global softmax bound]\label{rem:global-softmax-bound}
For any probability vector \(p\), \(\eig_{\max}(G(p))\leq1/2\).
Consequently, the population simplex model obeys
\(\Lambda_p(x)\leq d_\mu/(2\sqrt2\,\sigma^2)\).
On \(\R^d\), this bound is a supremum approached along active facets as the posterior tends to \((1/2,1/2,0,\ldots,0)\). 
At every finite input, all Gaussian posterior components are positive, so the bound is not attained. 
This global statement is consistent with local normal minima on parts of a facet.
\end{remark}
Equation~\eqref{eq:multiclass-full-spectrum} also gives every
probability-endpoint finite-time Lyapunov spectrum with \(\lambda_{i,p}(x)=\frac{1}{K+2}\log \sigma_i\!\left(Dp_{\theta^\star}(x)\right)\).

The boundary results below concern
the leading exponent: FTMLE. The simple form of \eqref{eq:multiclass-spectral-reduction} comes from the simplex identity \(VV^\top=(d_\mu^2/2)(I_C-\mathbf1\mathbf1^\top/C)\).
The output softmax covariance, rather than only the two tied posterior components, determines the expansion magnitude.

\begin{proposition}[Boundary-normal leading
direction]\label{prop:cclass-direction}
Under the assumptions of Proposition~\ref{prop:cclass-core}, let
\(x_0\in\mathring{\B}_{ij}^{\star}\), and write
\(\eta_i(x_0)=\eta_j(x_0)=a\), with
\(\eta_m(x_0)<a\) for \(m\notin\{i,j\}\). 
Then, the expansion \(\Lambda_p(x_0)=d_\mu a / ({\sqrt2\,\sigma^2})\).
Also,
\((e_i-e_j)/\sqrt2\) is the simple leading eigenvector of
\(G(\eta(x_0))\), up to sign. 
The corresponding leading right
singular vector of \(Dp_{\theta^\star}(x_0)\) and the FTMLE are, respectively,
\begin{equation*}
    v_1 \Big(Dp_{\theta^\star}(x_0)\Big) = \pm n_{ij}, \qquad 
  \lambda_p(x_0)=\frac{1}{K+2}\log\frac{d_\mu a}{\sqrt2\,\sigma^2}.
\end{equation*}
\end{proposition}

\noindent
The direction is normal on every smooth active facet, but this statement
does not say whether expansion decreases or increases upon leaving the
facet. The latter is a local magnitude question.

\begin{theorem}[Multiclass input-to-probability Gaussian simplex]\label{thm:cclass-criterion}
Under the assumptions of Proposition~\ref{prop:cclass-direction}, put
\(b_m\coloneqq \eta_m(x_0)<a\) for \(m\notin\{i,j\}\) and
\(b\coloneqq \sum_{m\notin\{i,j\}}b_m=1-2a\). For
\(x(s)=x_0+s n_{ij}\), set \(t=d_\mu s/(2\sigma^2)\). Then
\begin{equation}\label{eq:cclass-expansion}
  \eig_{\max}\bigl(G(\eta(x(s)))\bigr)
  =a+\frac{a}{2}\psi_{ij}(x_0)t^2+O(t^4),
  \qquad
  \psi_{ij}(x_0)
  \coloneqq 1-6a+
  \sum_{m\notin\{i,j\}}
  \frac{b_m(b+2b_m)}{a-b_m}.
\end{equation}
If \(\psi_{ij}(x_0)<0\), \(x_0\) is a strict normal maximum of
\(\lambda_p\); if \(\psi_{ij}(x_0)>0\), it is a strict normal
minimum. At \(\psi_{ij}(x_0)=0\), the first nonzero higher-order term
decides the local behavior.
\end{theorem}

\noindent
The sign criterion exhibits the key distinction: a leading singular
direction can point exactly across a decision facet even when the
facet is a local minimum of the leading expansion magnitude.

\begin{remark}
    The function \(\psi_{ij}(x_0)\) determines if the FTMLE attains its normal maxima or minima on the boundary. 
    The boundary as the minimizer of the FTMLE is a critical analytical result which, to the best of our knowledge, was not even empirically reported before.
    The interpretation is that, in a multi-class classification problem, FTMLE has a local extremum on the Bayes decision boundary, but this extremum can switch from a maximum to a minimum as the decision boundary is traversed.
\end{remark}

\begin{corollary}[Three-class
Gaussian simplex]\label{cor:threeclass-transition}
For \(C=3\), a smooth active facet has posterior
\((a,a,1-2a)\), up to permutation, with \(1/3<a<1/2\). Let
\begin{equation*}
  \ac\coloneqq \frac{\sqrt{57}-3}{12}\approx0.37915287.
\end{equation*}
The facet is a strict normal minimum for \(1/3<a<\ac\), a strict
normal maximum for \(\ac<a<1/2\), and a strict fourth-order normal
maximum at \(a=\ac\). Its leading input singular direction is
boundary-normal throughout \(1/3<a<1/2\).
\end{corollary}

The interpretation is that, in a multi-class classification problem, FTMLE has a local extremum on the Bayes decision boundary, but this extremum can switch from a maximum to a minimum as we traverse the decision boundary. This is illustrated in an example in Fig.~\ref{fig:multiclass-probability}, discussed in the next section.

Complete proofs of all results in Sec.~\ref{sec:multiclass-probability} are in Appendix~\ref{app:multiclass-core}.

\subsection{Numerical confirmation}\label{sec:probability-numerics}

We illustrate the results of the theorems above under finite-sample approximation.

\begin{figure}[t]
  \centering
  \includegraphics[width=0.8\linewidth]{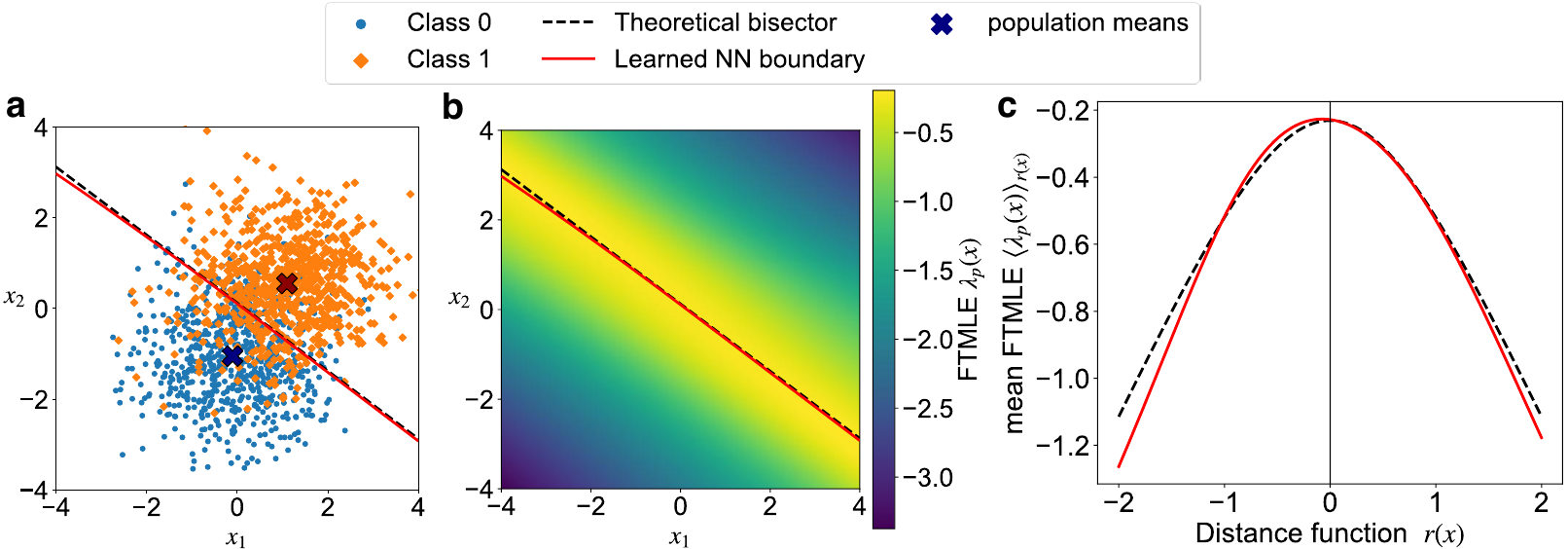}
  \caption{\textbf{Example of input-to-probability FTMLE of the binary classification.} 
  In all panels, the theoretical results are shown via a dashed black line, and the numerical evaluations of the learned NN model are shown via a solid red line.
  Panel a shows two 2D Gaussians, with the theoretical bisector and the learned NN boundary. 
  Panel b shows the variation of the input-to-probability FTMLE $\lambda_p (x)$ of the learned NN over the 2D space.
  Panel c compares the FTMLE $\langle\lambda_p (x)\rangle_{r(x)}$ of the learned NN (solid red line), which has been averaged over the signed-distance $r(x)$, with the analytically derived one (dashed black line).
  The vertical line at $r(x) = 0$ denotes the Bayes boundary.}
  \label{fig:binary-probability}
\end{figure}

\textbf{Binary classification example.}
For two classes in \(\R^2\), a \((2\)-\(4\)-\(1)\) tanh network was trained with binary cross-entropy on \(20{,}000\) samples. The independent test set contains \(30{,}000\) samples. 
Its posterior mean-squared error relative to the exact Bayes posterior is \(5.84\times10^{-5}\); test accuracy is \(0.83667\), compared with the empirical Bayes accuracy \(0.83690\) on the same test set. 
Figure~\ref{fig:binary-probability} compares the learned decision boundary via the theoretical bisector (panel a) and compares the averaged input-to-probability FTMLE gradient profile of the learned NN with the theoretical profile in Theorem~\ref{thm:binary-gaussian} (panels b and c). 
The Pearson profile correlation is \(0.99787\); the peak of the tangentially averaged learned profile is \(0.50684\), compared with the exact peak \(0.5\).
The mean absolute displacement between numerical learned-boundary roots and the Bayes bisector is \(0.02051\) in input units. 
Through a numerical exercise, we therefore showed the generality of our learning assumptions and numerically demonstrated our theoretical results for input-to-probability FTMLE analysis of binary classifications.

\begin{figure}[h]
  \centering
  \includegraphics[width=0.9\linewidth]{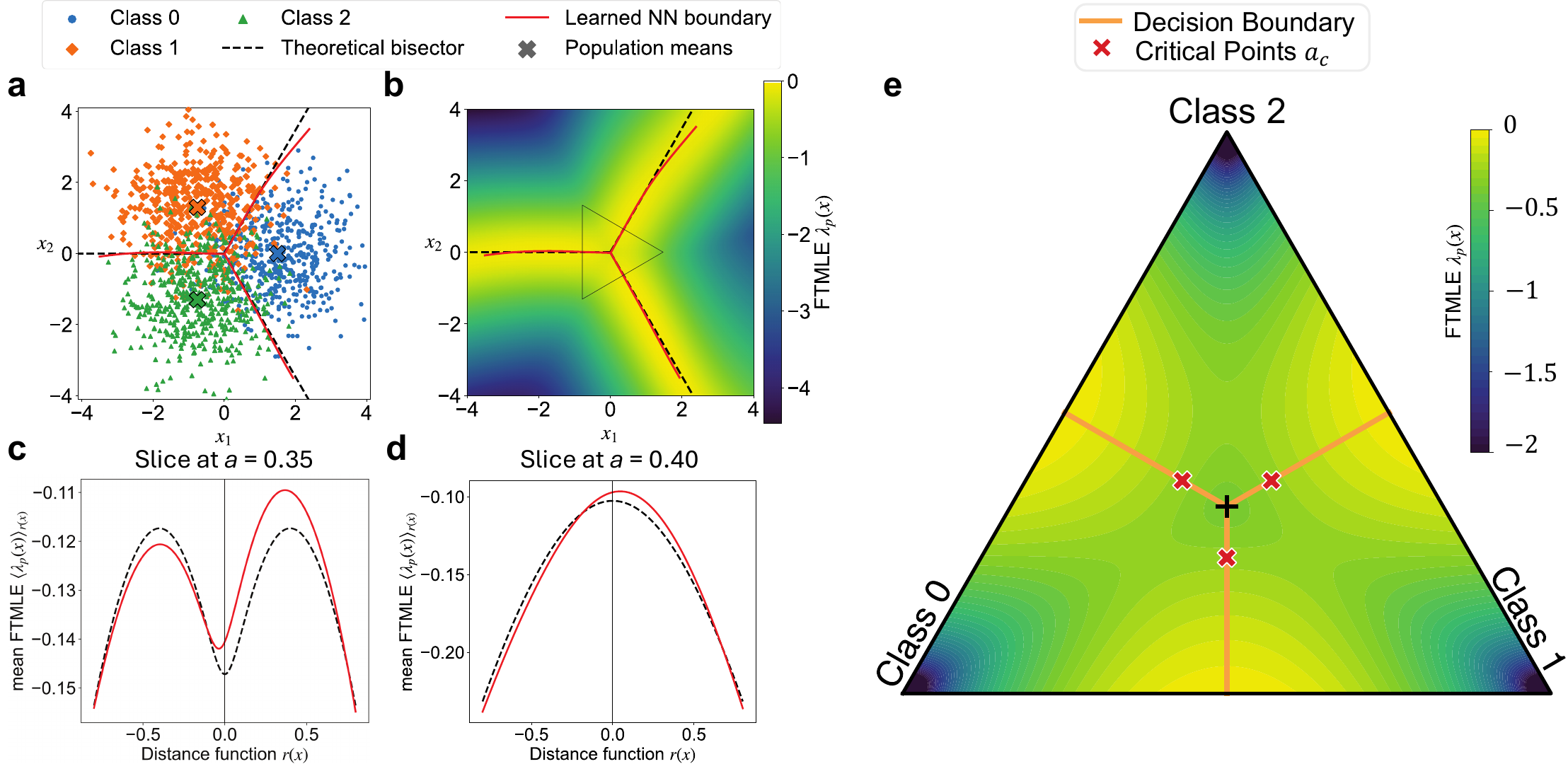}
  \caption{\textbf{Example of input-to-probability FTMLE of the multiclass classification.} 
  In panels a-d, the theoretical results are shown via a dashed black line, and the numerical evaluations of the learned NN model are shown via a solid red line.
  Panel a shows three 2D Gaussians, with the theoretical bisector and the learned NN boundary. 
  Panel b shows the variation of the input-to-probability FTMLE $\lambda_p (x)$ of the learned NN over the 2D space.
  Panel c compares the FTMLE $\langle\lambda_p (x)\rangle_{r(x)}$ of the learned NN (solid red line), which has been averaged over the signed-distance $r(x)$ on normal
  slices through $a = 0.35 < a_c$ on the active class-1/class-2 facet, with the analytically derived one (dashed black line).
  The vertical line at $r(x) = 0$ denotes the Bayes boundary.
  Panel d is similar to c, with the difference that the slice is at $a = 0.4 > a_c$.
  Panel e is the zoomed black triangle shown in Panel b.
  Here, the geometry of the simplex model is shown as an equilateral triangle, with its vertices located at the population means.
  A higher-resolution contour reveals that as one traverses along the decision boundary, beginning at the + mark, $a = 1/3$, the boundary is an FTMLE minimum ridge for $1/3 < a < a_c$, and the boundary is an FTMLE maximum ridge for $a > a_c$, with the critical point $a_c \approx 0.38$ (see Corollary~\ref{cor:threeclass-transition}).}
  \label{fig:multiclass-probability}
\end{figure}

\textbf{Multiclass classification example.}
For three regular-simplex classes in \(\R^2\), a \( (2\)-\(4\)-\(3) \)
tanh classifier was trained with categorical
cross-entropy on \(21{,}000\) samples and evaluated on \(60{,}000\)
independent test samples. Its posterior mean-squared error is
\(6.13\times10^{-5}\), and test accuracy is \(0.83725\), compared with
an empirical Bayes accuracy of \(0.83695\). 
Figure~\ref{fig:multiclass-probability} compares the learned decision boundary via the theoretical bisector (panel a) and compares the averaged input-to-probability FTMLE gradient profile of the learned NN with the theoretical profile in Theorem~\ref{thm:cclass-criterion} (panels b, c, and d). 
On a \(321 \times 321\)
Cartesian grid, Pearson correlations between learned and exact Jacobian
norm and unclipped FTMLE fields are \(0.98766\) and \(0.99519\),
respectively.
The higher-resolution probability simplex in Fig.~\ref{fig:multiclass-probability}\,e shows the results of Propositions~\ref{prop:cclass-core} and \ref{prop:cclass-direction} and Theorem~\ref{thm:cclass-criterion}, and marks the analytic transition points \(a=\ac\) on the active pairwise facets as red crosses.
Panels c and d of Fig.~\ref{fig:multiclass-probability} tests the two sides of
Corollary~\ref{cor:threeclass-transition}.  At \(a=0.35<\ac\), the analytic quadratic coefficient $\psi_{12} (x_0) >0$, so the boundary is a local minimum (Panel c). 
At \(a=0.40>\ac\), $\psi_{12} (x_0) <0$, so the boundary is a local maximum (Panel d).
Thus, the numerically evaluated solid red curves agree with the predicted minimum and maximum regimes in Corollary~\ref{cor:threeclass-transition}, shown as dashed black lines.

\section{Moving backward in time: input-to-logit
FTMLE}\label{sec:logit}

We now observe the cumulative dynamics at \(h_{K+1}=\ell_\theta\),
before the probability link. The output probability of a CE-trained
binary model can still have the ridge of
Theorem~\ref{thm:binary-gaussian}; the preceding logit can have a
different expansion profile. The raw logit also depends on the population
objective that determines it.

\subsection{MSE-trained binary logit}\label{sec:mse-logit}

Use the binary Gaussian geometry in~\eqref{eq:binary-geometry}, and set
\(\widehat{Y}\coloneqq 2Y-1\in\{-1,+1\}\). The raw-logit mean-squared error
(MSE) population risk is
  $R_{\mathrm{MSE}}(\theta)
  \coloneqq \E\!\left[
  \bigl(\widehat{Y}-\ell_\theta^{\mathrm{MSE}}(X)\bigr)^2
  \right].$

\begin{theorem}[MSE input-to-logit Gaussian classification]\label{thm:mse-logit}
Under Assumption~\ref{ass:gaussian} with \(C=2\) and
\(\Omega=\R^d\), assume the model class realizes the continuously differentiable
population MSE minimizer and training attains it. Then
\begin{equation*}
  \ell_{\theta^\star}^{\mathrm{MSE}}(x)
  =\E[\widehat{Y}\mid X=x]
  =2\eta(x)-1
  =\tanh\!\left(\frac{\alpha r(x)}{2}\right).
\end{equation*}
The boundary is \(\B^\star = \left\{ x : \ell_{\theta^\star}^{\mathrm{MSE}}(x) = 0 \right\}\), and
\begin{equation*}
  D\ell_{\theta^\star}^{\mathrm{MSE}}(x)
  =\frac{\alpha}{2}
   \operatorname{sech}^2\!\left(\frac{\alpha r(x)}{2}\right)n^\top,
  \quad
  \Lambda_\ell^{\mathrm{MSE}}(x)
  =\frac{\alpha}{2}
   \operatorname{sech}^2\!\left(\frac{\alpha r(x)}{2}\right),
  \quad
  \lambda_\ell^{\mathrm{MSE}}(x)
  =\frac{1}{K+1}\log\Lambda_\ell^{\mathrm{MSE}}(x).
\end{equation*}
The leading input singular direction is \(\pm n\), and
\[
  \operatorname*{arg\,max}_{x\in\R^d}
  \lambda_\ell^{\mathrm{MSE}}(x)=\B^\star.
\]
\end{theorem}

This ridge follows from the conditional mean implied by the MSE
objective; it does not follow from the sign boundary alone. In fact,
\(D\ell_{\theta^\star}^{\mathrm{MSE}}=2D\eta\), so the raw MSE-logit
and proper-loss binary probability maps have the same ridge location.
Their expansion magnitudes and depth normalizations differ.

\subsection{CE-trained binary logit}\label{sec:ce-logit}

For a cross-entropy (CE) model with sigmoid probability link, let
\begin{equation}\label{eq:ce-logit-risk}
  R_{\mathrm{CE}}(\theta)
  \coloneqq \E\!\left[
    -Y\log\sigm(\ell_\theta^{\mathrm{CE}}(X))
    -(1-Y)\log\{1-\sigm(\ell_\theta^{\mathrm{CE}}(X))\}
  \right].
\end{equation}

\begin{corollary}[CE-logit in the Gaussian
classification]\label{cor:ce-logit-constant}
Under Assumption~\ref{ass:gaussian} with \(C=2\) and
\(\Omega=\R^d\), assume the CE logit class realizes the Bayes log odds continuously and
the population risk in~\eqref{eq:ce-logit-risk} attains its global
minimum. In the binary Gaussian model,
\begin{equation*}
  \ell_{\theta^\star}^{\mathrm{CE}}(x)
  =\log\frac{\eta(x)}{1-\eta(x)}
  =\alpha r(x),\qquad
  \Lambda_\ell^{\mathrm{CE}}(x)=\alpha,\qquad
  \lambda_\ell^{\mathrm{CE}}(x)=\frac{\log\alpha}{K+1}.
\end{equation*}
The zero-logit boundary is \(\B^\star\), and the leading input
direction is \(\pm n\), but the logit FTMLE is constant on \(\R^d\).
\end{corollary}
Appendix~\ref{app:logit-proofs} gives the proofs of Theorem~\ref{thm:mse-logit} and Corollary~\ref{cor:ce-logit-constant}.
\begin{remark}
Our logit results are important as they show that the leading singular vector of the Jacobian at depth $K+1$ (logit level) may or may not point toward a higher FTMLE. Depending on the population risk, the FTMLE may be constant over the data space, or it may have a global maximum on the decision boundary.
\end{remark}

\section{Moving backward in time: input-to-representation
FTMLE}\label{sec:representation}

At \(h_K=h_\theta\), the cumulative map stops before the readout.
Its Jacobian maps input perturbations into every coordinate of the
learned representation. Only part of this change is visible to the
classification task. That distinction is central to interpreting a
hidden-layer FTMLE field.

\subsection{Hidden expansion, task-visible expansion, and task-alignment
ratio}\label{sec:hidden-decomposition}

For a binary affine readout with \(v\neq0\), let
  $\ell_\theta(x)=v^\top h_\theta(x)+\beta,$
and
  $\widehat v\coloneqq v/{\norm{v}_2}$,$P_v\coloneqq I-\widehat v\widehat v^\top$
be the unit readout and the orthogonal projector away from its direction.
Define the hidden expansion and the hidden FTMLE, respectively,
\begin{equation*}
  \Lambda_H(x)
  \coloneqq \norm{Dh_\theta(x)}_2,
  \qquad
  \lambda_H (x) = \dfrac{1}{K} \Lambda_H (x).
\end{equation*}
Now, based on the affine readout layer, we define the task-visible expansion:
\begin{equation*}
    \Lambda_T(x)
  \coloneqq \norm{\widehat v^\top Dh_\theta(x)}_2
  =\frac{\norm{D\ell_\theta(x)}_2}{\norm{v}_2} \leq \Lambda_H(x),
\end{equation*}
The two expansions \(\Lambda_T(x)\) and \(\Lambda_H(x)\) have different interpretations. 
For logit \(\ell_\theta(x)\ne0\) and \(D\ell_\theta(x)\ne0\), the minimum norm perturbation $\xi$ to the data, satisfying the local boundary equation \(\ell_\theta(x)+D\ell_\theta(x)\xi=0\), is the linear robustness radius
\begin{equation}\label{eq:linearized-radius}
\rho_{\mathrm{lin}}(x)
  = \frac{|\ell_\theta(x)|}
  {\norm{v}_2\Lambda_T(x)}.
\end{equation}
A larger $\rho_{\mathrm{lin}}(x)$ means more robustness to data perturbations, which is related to task-visible expansion $\Lambda_T (x)$.
In contrast, a readout perturbation \(\xi_v\) affects the readout weights by \(v+\xi_v\) and changes the logit derivative by \(\xi_v^\top D h_\theta(x)\), whose norm is at most \(\norm{\xi_v}\Lambda_H(x)\); this upper bound is attained by a perturbation along the leading left singular vector of $Dh_\theta (x)$ and is controlled by the hidden expansion $\Lambda_H (x)$.
Thus a large \(\Lambda_H\) can indicate
input directions that a small change of readout could make visible,
even if those directions currently have little effect on the logit.

Whenever \(\Lambda_H(x)>0\), define the task-alignment
\begin{equation*}
  \chi(x)\coloneqq \frac{\Lambda_T(x)}{\Lambda_H(x)}\in[0,1].
\end{equation*}
The ratio \(\chi\) compares task-visible and hidden expansion.
The higher \(\chi(x)\), the more expansions in the hidden representation are transferred to the readouts; thus, the task-visibility of the readout layer.
For every unit input direction \(u\), orthogonality yields
\begin{equation}\label{eq:hidden-orthogonal-identity}
  \norm{Dh_\theta(x)u}_2^2
  =\bigl|\widehat v^\top Dh_\theta(x)u\bigr|^2
   +\norm{P_vDh_\theta(x)u}_2^2.
\end{equation}
Consequently
\begin{equation*}
  \Lambda_T(x)\leq\Lambda_H(x)
  \leq
  \sqrt{\Lambda_T(x)^2+\norm{P_vDh_\theta(x)}_F^2}.
\end{equation*}
When \(P_vDh_\theta(x) = 0\), then \(\chi(x) = 1\), i.e., when there is no hidden expansion invisible to the readout layer, the task-alignment ratio attains its max value $1$.

\subsection{Attaining max hidden FTMLE on the ideal decision boundary is non-trivial}\label{sec:representation-information}

\begin{example}[Task-orthogonal off-boundary
expansion]\label{ex:hidden-off-boundary}
Let \(x=(x_1,x_2)\in\R^2\), and assume the ideal decision boundary is \(\B=\{x:x_1=0\}\). 
Choose
\(\alpha,\gamma,\nu>0\) and \(s_0\neq0\), and define
\begin{equation}\label{eq:hidden-example-map}
  h_\theta(x)
  =
  \begin{bmatrix}
    \tanh(\alpha x_1)\\
    \tanh\big(\nu(x_1-s_0)\big)
  \end{bmatrix},
  \qquad
  \ell_\theta(x)=\gamma\tanh(\alpha x_1).
\end{equation}
Thus, the classifier learns the ideal boundary.
However, the readout ignores the second coordinate, which we will show results in hidden expansion \(\Lambda_H(x)\) being invisible to the logit.
Here,
\begin{equation}\label{eq:hidden-example-expansion}
  \Lambda_T(x)=\alpha\operatorname{sech}^2(\alpha x_1),
  \quad
  \Lambda_H(x)^2
  =\alpha^2\operatorname{sech}^4(\alpha x_1)
  +\nu^2\operatorname{sech}^4\big(\nu(x_1-s_0)\big).
\end{equation}
The task-visible expansion \(\Lambda_T(x)\) has its global maximum on \(\B\), while
\(\Lambda_H\) is not even stationary there. In fact,
\[
  \left.\frac{d}{dx_1}\Lambda_H(x)^2\right|_{x_1=0}
  =4\nu^3\operatorname{sech}^4(\nu s_0)
    \tanh(\nu s_0)\neq0.
\]
The hidden expansion attains a global maximum on a line
\(x_1=s_\star\neq0\). This gives a concrete representation-level
pattern that the scalar logit cannot show.
\end{example}

\noindent
For \(\alpha=1,\nu=2,s_0=1,\gamma=1\), the exact formulas give
\((\Lambda_T,\Lambda_H)=(1,1.01)\) on \(x_1=0\) and
\((0.42, 2.04)\) on \(x_1=1\), which shows that, although the task-visible expansion $\Lambda_T$ is max on the boundary $x_1 = 0$, the hidden expansion $\Lambda_H (x)$ is not.

\subsection{Sufficient ridge transfer conditions}\label{sec:transfer-conditions}

A logit FTMLE ridge and a task-visible
expansion ridge have the same location. To transfer this location to the
hidden expansion, one needs control of the directions orthogonal to
the readout. 

\begin{proposition}
[Task-aligned hidden FTMLE localization]\label{prop:task-aligned-localization}
Theorem~\ref{thm:mse-logit} gives
\(
    \lambda_\ell(x)<\lambda_\ell(\pi(x))
\)\(
    \ \Longleftrightarrow \
    \Lambda_T(x)<\Lambda_T(\pi(x)),
\) \(
     x\notin\B^\star,
\)
for a projection \(\pi:\Omega\to\B^\star\). To transfer this ridge to the hidden representation, it is sufficient to impose the additional task-alignment condition \(\Lambda_H(x)<\Lambda_T(\pi(x))\) for \(x\notin\B^\star\).
\end{proposition}
Motivated by the sufficient condition for logit ridge transfer to the representation layer, we propose geometry-aware fine-tuning in the following.

\subsection{Geometry-aware fine-tuning
objectives}\label{sec:fine-tuning}

Let \(\B\coloneqq \{x\in\Omega:\ell_\theta(x)=0\}\). On a tubular neighborhood of a
smooth boundary patch, let \(\pi(x)\in\B\) be the unique nearest-point
projection and let \(d(x)\coloneqq \norm{x-\pi(x)}_2\). For \(\kappa>0\),
\(d(x)>0\), \(\Lambda_T(x)>0\), and \(\Lambda_T(\pi(x))>0\), define
\begin{align*}
  r_T(x)
  \coloneqq 
  \frac{\left[
    \kappa d(x)^2+
    \log\!\left(\frac{\Lambda_T(x)}{\Lambda_T(\pi(x))}\right)
  \right]_+}
  {\kappa d(x)^2},\qquad 
  r_\perp(x)
  \coloneqq 
  \frac{\norm{P_vDh_\theta(x)}_F^2}
  {d(x)^2\Lambda_T(x)^2},
\end{align*}
where \([s]_+\coloneqq \max\{s,0\}\).

For a nonempty finite set \(\mathcal X_O\) of admissible off-boundary points paired with their projections, define
\begin{equation*}
  \mathcal L_T
  \coloneqq \frac{1}{|\mathcal X_O|}
    \sum_{x\in\mathcal X_O}r_T(x)^2,
  \qquad
  \mathcal L_\perp
  \coloneqq \frac{1}{|\mathcal X_O|}
    \sum_{x\in\mathcal X_O}r_\perp(x).
\end{equation*}
Small residuals of $\mathcal L_T$ encourage task-visible expansion to decrease normally away from the boundary, concentrating it where class separation is required (at the boundary). 
The loss $\mathcal L_\perp$ decreases nuisance expansion invisible to the logit, which prohibits the hidden expansion from growing strongly in directions that are almost invisible to the final classifier.

If \(\theta_0\) denotes a
frozen pretrained model, let \(\mathcal X\) collect the boundary and
interior points used for preservation and define
\begin{equation*}
  \mathcal L_{\mathrm{pres}}
  \coloneqq \frac{1}{|\mathcal X|}
    \sum_{x\in\mathcal X}
    \bigl(\ell_\theta(x)-\ell_{\theta_0}(x)\bigr)^2.
\end{equation*}
Low values of $\mathcal L_{\mathrm{pres}}$ preserves the learned decision boundary. 
The numerical study below adds the empirical fine-tuning objective
\begin{equation*}
\mathcal L_{\mathrm{fine}}=\mathcal L_{\mathrm{task}}+\alpha_T\mathcal L_T+\alpha_\perp\mathcal L_\perp+\alpha_{\mathrm{pres}}\mathcal L_{\mathrm{pres}}, \qquad \alpha_T , \alpha_\perp, \alpha_{\mathrm{pres}} > 0,
\end{equation*}
added to the pre-training loss.
Next, we provide residual conditions in which we guarantee a normal hidden FTMLE ridge under task-aligned
bounds.
\begin{theorem}[Fine-tuning input-to-representation classification]\label{thm:hidden-ridge}
Let \(h_\theta\in C^1(\Omega;\R^q)\), \(v\neq0\), and
\(\B_0\subseteq\B\) be a smooth boundary patch with tubular neighborhood
\[
  \mathcal{U}_{\varepsilon_0}
  \coloneqq \{b+t\,n(b):b\in\B_0,\ |t|<\varepsilon_0\},
\]
contained in \(\Omega\), where every \(x=b+t\,n(b)\) has unique
projection \(\pi(x)=b\), and assume
\(\Lambda_T>0\) on \(\mathcal{U}_{\varepsilon_0}\). Suppose, uniformly on
\(\mathcal{U}_{\varepsilon_0}\setminus\B_0\),
\[
  r_T(x)\leq\varepsilon_T,\qquad 0\leq\varepsilon_T<1,
  \qquad
  r_\perp(x)\leq\varepsilon_\perp,\qquad\varepsilon_\perp\geq0,
  \qquad
  \frac{\varepsilon_\perp}{2}
  <(1-\varepsilon_T)\kappa.
\]
Then, for every \(b\in\B_0\) and every \(0<|t|<\varepsilon_0\),
\[
  \lambda_H\bigl(b+t\,n(b)\bigr)<\lambda_H(b).
\]
Hence \(\B_0\) is a strict normal hidden-representation FTMLE ridge.
\end{theorem}

The theorem is relative to the \emph{current} logit boundary \(\B\).
Moreover, the uniform
\(r_\perp\)-bound forces
\(P_vDh_\theta(b)=0\) at every \(b\in\B_0\).
Appendix~\ref{app:hidden-ridge} gives the complete proof.

\subsection{Numerical illustration}\label{sec:representation-numerics}

We examine whether the proposed fine-tuning objectives concentrate hidden sensitivity near the decision boundary, as motivated by Theorem~\ref{thm:hidden-ridge}. The example uses two moons dataset \citep{JMLR:v12:pedregosa11a} with independent training, validation, and test sets of $10{,}000$, $5{,}000$, and $5{,}000$ samples. A \( (4 \)-\(4\)-\(4\)-\(4) \) Softsign network with a scalar linear readout is pretrained for $500$ epochs using binary cross-entropy (CE) and Adam with learning rate $10^{-4}$. Two copies of this pretrained classifier are then trained for $1{,}000$ further epochs: one continues CE, and the other combines CE with $\mathcal L_{\mathrm{fine}}$. Both continuations use fresh Adam optimizers and a learning rate of $3\times10^{-4}$; all phases use minibatches of size $256$. The geometry-aware objective uses $\kappa=0.1$, unit CE weight, and fixed loss weights calibrated before fine-tuning. Moreover, $\mathcal L_T$ remained zero throughout recorded fine-tuning. 

\begin{figure}[t]
    \centering
    \includegraphics[width=0.9\linewidth]{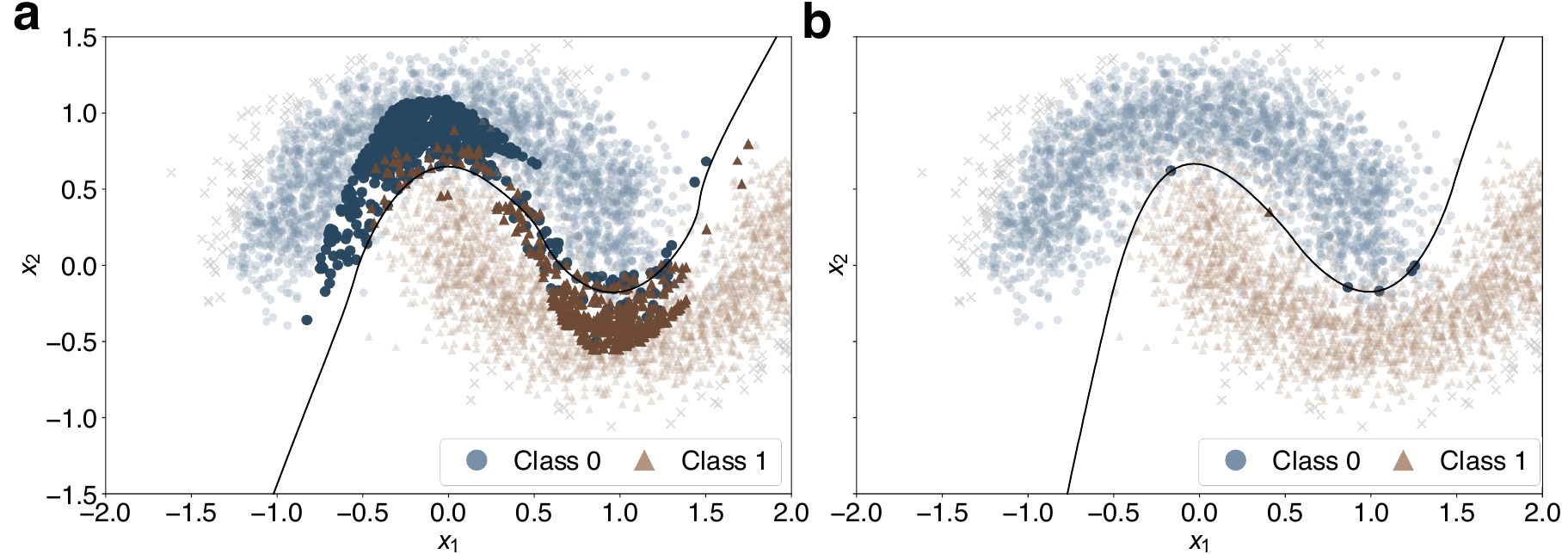}
    \caption{\textbf{Example of geometry-aware fine-tuning.} a: continued CE. b: geometry-aware fine-tuning. Darker markers identify samples with hidden expansion exceedance than their projected point on the boundary (points with $R_H(x) >10^{-6}$). The black curve denotes the decision boundary. 
    In panel b, only 6 exceedances are observed with a very small distance to the boundary (order $10^{-4}$), while 924 samples in panel a have hidden expansion exceedances.}
    \label{fig:TwoMoons}
\end{figure}

In Fig.\,\ref{fig:TwoMoons},
we evaluate the hidden expansion exceedance of off-boundary points relative to their projected point on the boundary by $R_H(x)=\log\Lambda_H(x)-\log\Lambda_H(\pi(x))$.
Among $4{,}813$ test samples, hidden-expansion exceedances (points with $R_H(x) >10^{-6}$) from the boundary decrease from $924$ samples ($19.20\%$) in panel a to $6$ samples ($0.12\%$) in panel b. 
Thus $99.9\%$ of the assessed samples satisfy the desired hidden-expansion ordering within numerical tolerance.
This shows the significance of our proposed geometry-aware fine-tuning in suppressing the nuisance expansions away from the boundary.

Figure~\ref{fig:fine-tuning}\,a-b examines the residuals defined in Section~\ref{sec:fine-tuning}. Geometry-aware fine-tuning reduces both $r_T$ and $r_\perp$ residuals over normal offsets. 
We evaluate the margin of the rate condition of Theorem~\ref{thm:hidden-ridge} by $\Delta_{\mathrm{sample}}=(1-\max r_T)\kappa-\max r_\perp/2$.
The geometry-aware fine-tuning results in a margin $\Delta_{\mathrm{sample}} = 0.0904 >0$ (satisfaction of the rate condition), compared with $\Delta_{\mathrm{sample}} = -29.0 <0$ for continued CE (failure of the rate condition); see Fig.~\ref{fig:fine-tuning}\,c.  This check supports the proposed mechanism, in which useful expansions are concentrated near the boundary, and nuisance expansions are suppressed.
The corresponding expansion profiles show the geometric effect directly in panels d and e. Both $\Lambda_H$ and $\Lambda_T$ decay more strongly away from the boundary after geometry-aware fine-tuning, with negative log ratios on both sides. Across off-boundary normal samples, the fraction with hidden expansion exceeding its boundary value falls from $44.0\%$ to $0.70\%$. Since the leading FTMLE is a positive depth-scaled logarithm of expansion, these relative orderings carry over to the hidden and logit endpoints.
Geometry-aware fine-tuning improved robustness: the distribution of $\rho_{\mathrm{lin}}$ shifts toward larger values; its median increases from $0.95$ to $2.18$, as shown in panel f.
Test accuracy is preserved through our $\mathcal{L}_\mathrm{pres}$ loss, with $96.96\%$ in pre-training, compared with $96.94\%$ for continued CE and $96.96\%$ for our geometry-aware fine-tuning.

\begin{figure}[t]
    \centering
    \includegraphics[width=\linewidth]{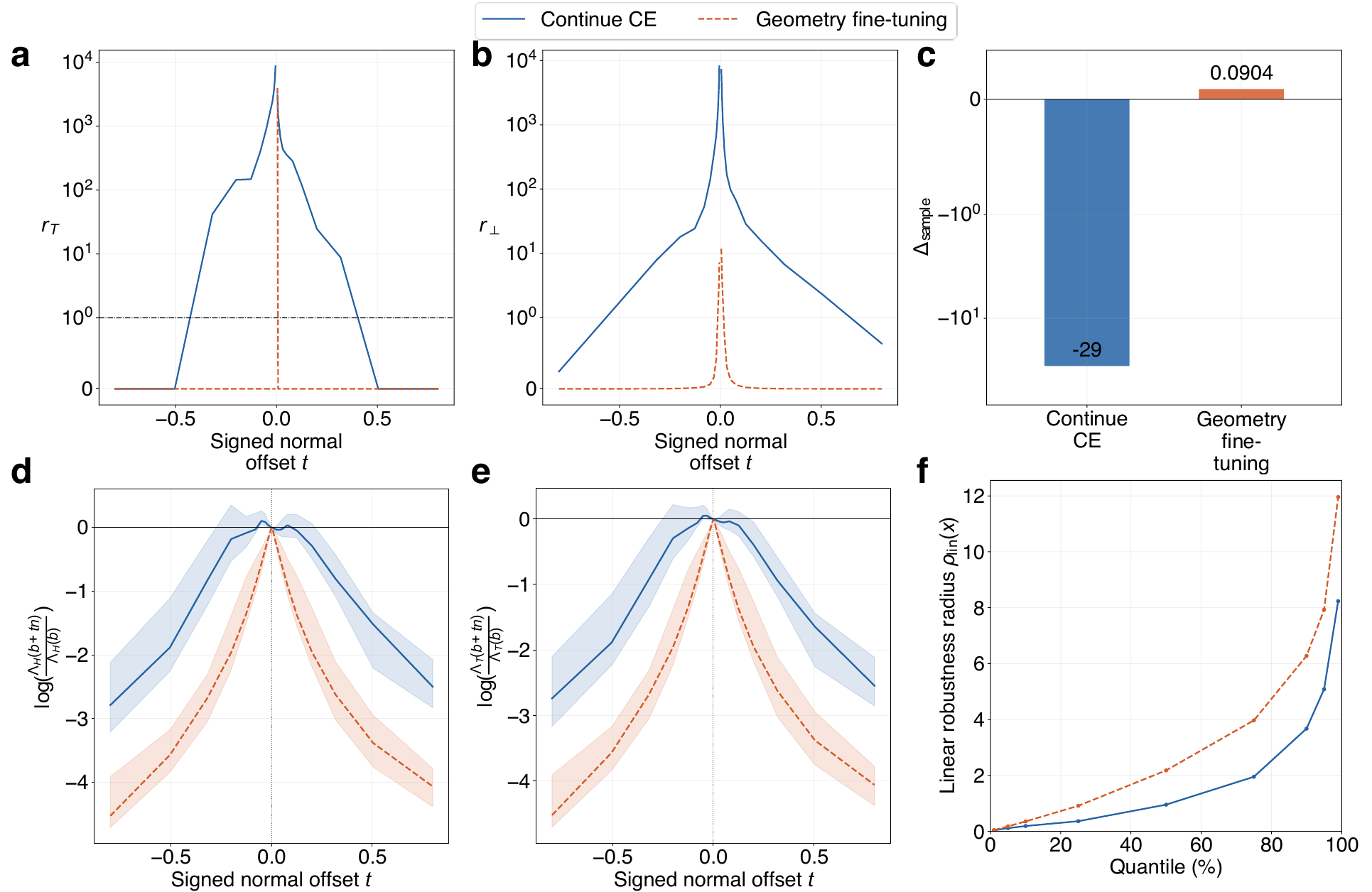}
    \caption{\textbf{Example of geometry-aware fine-tuning: Residual diagnostics, hidden and task-aligned expansion profiles, and robustness.} In all panels, continued cross-entropy and geometry-aware fine-tunings are shown as solid blue and dashed red lines, respectively.
    a: task-visible residual $r_T$ and b: nuisance residual $r_\perp$, across boundary-normal samples at each signed offset $t$. The black horizontal reference on panel a is $r_T=1$ (Theorem~\ref{thm:hidden-ridge} requirement). c: Only a positive value supports the rate condition of Theorem~\ref{thm:hidden-ridge}. d and e: Curves show median values of the log-ratio between hidden and task-aligned expansion normal to the boundary (more negative values away from $t=0$ indicate a desired decrease in expansion further from the boundary). Shaded bands indicate the $10$th--$90$th percentiles. f: linear robustness radius is increased (as desired) through geometry-aware fine-tuning.}
    \label{fig:fine-tuning}
\end{figure}

\section{Conclusions}\label{sec:conclusions}

A decision boundary is not necessarily a maximal FTMLE ridge: it can form a minimum ridge, or have no distinctive FTMLE magnitude at all. By following the observation endpoint backward from probabilities to logits and hidden representations, our theoretical results establish when finite-time expansion identifies decision geometry and when that interpretation fails.

At the probability endpoint, Theorem~\ref{thm:binary-gaussian} establishes an exact boundary-maximal FTMLE profile for population-optimal binary classification under the stated Gaussian assumptions, with the leading input singular direction normal to the boundary. Proposition~\ref{prop:binary-probability} extends this mechanism locally to smooth boundaries when the recovered posterior has the prescribed signed-distance profile. For regular-simplex multiclass Gaussian models, Propositions~\ref{prop:cclass-core} and~\ref{prop:cclass-direction} characterize the full probability-Jacobian spectrum and establish normal alignment on smooth active pairwise facets. Crucially, Theorem~\ref{thm:cclass-criterion} shows that these facets can be either normal maxima or normal minima of FTMLE, with a sharp three-class transition given by Corollary~\ref{cor:threeclass-transition}. To the best of our knowledge, the identification and numerical demonstration of such a boundary minimum ridge are novel. This result separates two properties often implicitly conflated: the leading expansion direction can point exactly across the boundary while its magnitude is locally smallest there.

The logit results further show that this relationship depends on the training objective. In the same binary Gaussian setting, Theorem~\ref{thm:mse-logit} proves that population-optimal MSE training against signed labels produces boundary-maximal logit FTMLE, whereas Corollary~\ref{cor:ce-logit-constant} proves that population-optimal CE training produces an affine logit with spatially constant FTMLE. Both recover the same decision boundary. Thus, even exact recovery of the classification geometry does not determine the spatial sensitivity profile, and a probability-level ridge may not persist at the preceding logit endpoint.

At the hidden endpoint, our counterexample demonstrates that expansion invisible to the readout can displace the maximum away from the boundary, even when task-visible expansion is boundary maximal. Proposition~\ref{prop:task-aligned-localization} identifies a sufficient condition for transferring localization to the full hidden representation. Building on this distinction, our geometry-aware fine-tuning encourages task-visible expansion to decrease away from the boundary, suppresses expansion orthogonal to the readout, and penalizes changes to the pretrained logits. Theorem~\ref{thm:hidden-ridge} ensures a strict normal hidden-FTMLE ridge along the learned boundary patch. This provides explicit conditions under which the desired hidden expansion profile can be established.

Our numerical examples support the validity and practical relevance of the theory and its assumptions in the tested settings. Finite-sample Gaussian experiments reproduce the predicted binary profile and multiclass minimum-to-maximum transition, while the nonlinear two-moons example demonstrates the applicability of the representation framework beyond Gaussian class geometry. Geometry-aware fine-tuning concentrates hidden expansion near the learned boundary and suppresses it away from the boundary: approximately \(99.9\%\) of the assessed test samples have hidden expansion no greater than at their boundary projections, within numerical tolerance. Exceedances fall from \(924\) test samples under continued CE to only \(6\) test samples, while retaining the pretrained test accuracy of \(96.96\%\). The normal profiles and positive sampled rate margin further support the theorem's mechanism.

These findings make a theoretical foundation essential for interpreting derivative-based analyses. Whenever FTMLE or Jacobian-norm sensitivity is used to locate decision boundaries, the assumption that it must peak near the boundary requires examination: our results exhibit both its validity and explicit failures. The observation endpoint, training objective, posterior geometry, and alignment with the readout must therefore inform the analysis. Used with this knowledge and care, sensitivity becomes a principled tool for understanding and shaping decision geometry.

\bibliography{Ref2}

\newpage
\appendix

\section{Statistical preliminaries and proof of proper-loss
recovery}\label{app:statistical-details}

\begin{proof}[Proof of Lemma~\ref{lem:proper-loss}]
Fix an input \(x\). In the binary case let \(\eta=\Pp(Y=1\mid
X=x)\) and let \(p\in[0,1]\) be a prediction. The conditional squared
probability risk is
\[
  \eta(1-p)^2+(1-\eta)p^2
  =\eta(1-\eta)+(p-\eta)^2,
\]
so its unique minimum is \(p=\eta\). For cross-entropy, using the
usual extended-real convention at \(p=0,1\),
\[
  -\eta\log p-(1-\eta)\log(1-p)
  =H_{\mathrm b}(\eta)
   +D_{\mathrm{KL}}\bigl(
     \operatorname{Bern}(\eta)\|
     \operatorname{Bern}(p)\bigr).
\]
Gibbs' inequality gives a unique minimum at \(p=\eta\), including
the endpoint cases.

For \(C\) classes, let \(\eta=(\eta_0,\ldots,\eta_{C-1})^\top\) and
let \(p\) range over the probability simplex. Categorical
cross-entropy is
\[
  -\sum_{i=0}^{C-1}\eta_i\log p_i
  =H(\eta)+D_{\mathrm{KL}}(\eta\|p),
\]
again uniquely minimized at \(p=\eta\). If \(e_Y\) is the one-hot
class vector, the conditional Brier risk satisfies
\[
  \E[\norm{e_Y-p}_2^2\mid X=x]
  =1-\norm{\eta}_2^2+\norm{p-\eta}_2^2,
\]
and has the same unique minimizer.

Integrating conditional risks over \(X\), global population
optimality and realizability imply \(p_{\theta^\star}(X)=\eta(X)\)
almost surely. Suppose the input density is strictly positive on
an open domain and both maps are continuous there. If the maps
differed at any point, continuity would make them differ on an open
neighborhood of positive input probability, contradicting the
almost-sure equality. Thus they agree pointwise. Their \(C^1\)
regularity then permits equality of derivatives on that domain.
\end{proof}

\section{Proofs for binary input-to-probability
results}\label{app:binary-probability}

\begin{proof}[Proof of Theorem~\ref{thm:binary-gaussian}]
By equal priors and the common spherical covariance,
\begin{align*}
  \log\frac{\eta(x)}{1-\eta(x)}
  &=\log\frac{f_{X\mid Y=1}(x)}{f_{X\mid Y=0}(x)}\\
  &=\frac{\norm{x-\mu_0}_2^2-\norm{x-\mu_1}_2^2}
          {2\sigma^2}
   =\frac{\delta^\top(x-c)}{\sigma^2}
   =\alpha r(x).
\end{align*}
Therefore \(\eta(x)=\sigm(\alpha r(x))\). The Gaussian mixture
density is positive on \(\R^d\), so
Lemma~\ref{lem:proper-loss} and Assumption~\ref{ass:learning}
give the pointwise identity \(p_{\theta^\star}=\eta\). Its
\(1/2\)-level set is \(\B^\star=\{r=0\}\).

Because \(\nabla r=n\) and
\(\sigm'(s)=\tfrac14\operatorname{sech}^2(s/2)\),
\[
  Dp_{\theta^\star}(x)
  =\frac{\alpha}{4}
   \operatorname{sech}^2\!\left(\frac{\alpha r(x)}2\right)n^\top.
\]
The nonzero scalar factor makes this a rank-one row Jacobian
with operator norm equal to that factor and right singular
direction \(\pm n\). Its norm is maximized precisely when
\(r(x)=0\), because
\(\operatorname{sech}^2(s)<1\) for \(s\neq0\).
Taking the increasing logarithm and dividing by the positive depth
\(K+2\) preserves the maximizer set.
\end{proof}

\begin{proof}[Proof of Proposition~\ref{prop:binary-probability}]
Posterior recovery and the chain rule give
\[
  Dp_{\theta^\star}(x)
  =\alpha\sigm'(\alpha r(x))\nabla r(x)^\top
  =\frac{\alpha}{4}
    \operatorname{sech}^2\!\left(\frac{\alpha r(x)}2\right)
    \nabla r(x)^\top.
\]
Since \(r\) is signed distance,
\(\norm{\nabla r(x)}_2=1\) in the tubular neighborhood. This
proves the expansion profile and the leading direction.
At a boundary point \(x_0\), the normal-coordinate property of
signed distance gives
\(r(x_0+t\nabla r(x_0))=t\) for sufficiently small \(t\).
The strictly peaked \(\operatorname{sech}^2\) factor therefore
gives a strict normal maximum. If all stated hypotheses hold
throughout \(\Omega\), the same profile is globally maximal
precisely at points with \(r=0\). The assumed nonempty
\(\B^\star\cap\Omega\) makes this set attainable.
\end{proof}

\section{Proofs for multiclass input-to-probability
results}\label{app:multiclass-core}

\begin{proof}[Proof of Proposition~\ref{prop:cclass-core}]
Expanding the Gaussian exponent around \(c\) gives
\[
 -\frac{\norm{x-\mu_i}_2^2}{2\sigma^2}
 =-\frac{\norm{x-c}_2^2}{2\sigma^2}
  +\frac{\nu_i^\top(x-c)}{\sigma^2}
  -\frac{\norm{\nu_i}_2^2}{2\sigma^2}.
\]
The first term is common to all classes; so is the last, because
all centered means have norm \(\rho\). Equal priors and
Lemma~\ref{lem:proper-loss} therefore give
Equation~\eqref{eq:simplex-posterior}. The pairwise log posterior
ratio is
\[
  \log\frac{\eta_i(x)}{\eta_j(x)}
  =\frac{(\nu_i-\nu_j)^\top(x-c)}{\sigma^2}
  =\frac{d_\mu}{\sigma^2}r_{ij}(x).
\]
The two posteriors tie on the perpendicular bisector \(r_{ij}=0\);
the maximal-posterior condition restricts this hyperplane to its
active Voronoi facet. If a softmax logit realizes \(\eta\), its
pairwise difference equals this log ratio, regardless of the common
logit gauge.

Let \(V\) be the matrix of centered means and let \(\mathbf1\) be
the \(C\)-vector of ones. The simplex Gram identities imply
\[
  VV^\top
  =\frac{d_\mu^2}{2}
   \left(I_C-\frac{\mathbf1\mathbf1^\top}{C}\right).
\]
Differentiating the posterior gives
\(D\eta(x)=\sigma^{-2}G(\eta(x))V\). Since
\(G(\eta(x))\mathbf1=0\),
\[
  D\eta(x)D\eta(x)^\top
  =\frac{1}{\sigma^4}G(\eta)VV^\top G(\eta)
  =\frac{d_\mu^2}{2\sigma^4}G(\eta)^2.
\]
The matrix \(G(\eta)\) is symmetric positive semidefinite and, for
strictly positive Gaussian posteriors, has rank \(C-1\). Its
nonzero eigenvalues therefore give the nonzero singular values of
\(D\eta\), proving~\eqref{eq:multiclass-full-spectrum} and, in
particular,~\eqref{eq:multiclass-spectral-reduction}.
\end{proof}

\begin{proof}[Proof of Remark~\ref{rem:global-softmax-bound}]
For any unit \(w\in\R^C\),
\(w^\top G(p)w\) is the variance of a discrete random variable
whose values are \(w_0,\ldots,w_{C-1}\) with probabilities \(p_i\).
The variance is at most one fourth of the squared range, while
\(\max_{i,j}|w_i-w_j|\leq\sqrt2\norm{w}_2=\sqrt2\).
Hence \(w^\top G(p)w\leq1/2\), proving the spectral bound.
Equality requires probability \(1/2\) on each of the two coordinates
at the two extremes and zero elsewhere.

For fixed \(i\neq j\), take \(x_s=c+s(\nu_i+\nu_j)\) as
\(s\to\infty\). The two selected Gaussian discriminants are
equal, and for every \(m\notin\{i,j\}\),
\[
  (\nu_i-\nu_m)^\top(x_s-c)
  =(\nu_j-\nu_m)^\top(x_s-c)
  =\frac{C\rho^2}{C-1}s>0.
\]
Thus \(\eta(x_s)\to(e_i+e_j)/2\) and
\(\eig_{\max}(G(\eta(x_s)))\to1/2\).
At any finite input, every Gaussian posterior component is positive,
so equality cannot occur.
\end{proof}

\begin{proof}[Proof of Proposition~\ref{prop:cclass-direction}]
Put \(p=\eta(x_0)\) and \(q=(e_i-e_j)/\sqrt2\). Since
\(p_i=p_j=a\), we have \(p^\top q=0\) and
\[
  G(p)q=\diag(p)q-pp^\top q=aq.
\]
For any unit vector \(w\in\R^C\),
\[
  w^\top G(p)w
  =\sum_{m=0}^{C-1}p_mw_m^2-(p^\top w)^2
  \leq\sum_{m=0}^{C-1}p_mw_m^2
  \leq a.
\]
Because only coordinates \(i,j\) attain the maximal posterior
value \(a\), equality requires \(w\) to be supported on this pair.
It also requires \(p^\top w=0\), so \(w=\pm q\).
Thus \(a\) is the simple leading eigenvalue of \(G(p)\).

The simplex geometry gives
\(V^\top q=(\nu_i-\nu_j)/\sqrt2=(d_\mu/\sqrt2)n_{ij}\).
Consequently
\[
  Dp_{\theta^\star}(x_0)^\top q
  =\frac{1}{\sigma^2}V^\top G(p)q
  =\frac{d_\mu a}{\sqrt2\,\sigma^2}n_{ij}.
\]
By the identity for
\(Dp_{\theta^\star}Dp_{\theta^\star}^\top\) proved above,
\(q\) is the leading left singular vector and \(n_{ij}\) is
the leading right singular vector. The singular value is
\(d_\mu a/(\sqrt2\,\sigma^2)\), as claimed.
\end{proof}

\begin{proof}[Proof of Theorem~\ref{thm:cclass-criterion}]
Set \(p(t)=\eta(x_0+s n_{ij})\), with
\(t=d_\mu s/(2\sigma^2)\). The regular-simplex identities give
\(\nu_i^\top n_{ij}=d_\mu/2\),
\(\nu_j^\top n_{ij}=-d_\mu/2\), and
\(\nu_m^\top n_{ij}=0\) for \(m\notin\{i,j\}\).
Equation~\eqref{eq:simplex-posterior} therefore yields
\begin{equation*}
  p_i(t)=\frac{ae^t}{Z(t)},\quad
  p_j(t)=\frac{ae^{-t}}{Z(t)},\quad
  p_m(t)=\frac{b_m}{Z(t)},\quad
  Z(t)=2a\cosh t+b.
\end{equation*}
Changing \(t\) to \(-t\) exchanges only the \(i,j\) coordinates.
Thus \(G(p(-t))\) is permutation-similar to \(G(p(t))\).
The top eigenvalue at \(t=0\) is the simple eigenvalue \(a\) by
Proposition~\ref{prop:cclass-direction}; it has a locally analytic,
even continuation. In particular, its odd Taylor coefficients vanish.

Write \(p(t)=p^{(0)}+tp^{(1)}+t^2p^{(2)}+O(t^3)\), where
\[
  p_i^{(1)}=a,\quad p_j^{(1)}=-a,\quad
  p_i^{(2)}=p_j^{(2)}=a\left(\frac12-a\right),\quad
  p_m^{(2)}=-ab_m\quad(m\notin\{i,j\}).
\]
Accordingly,
\(G(p(t))=G_0+tG_1+t^2G_2+O(t^3)\), with
\begin{align*}
  G_0&=\diag(p^{(0)})-p^{(0)}(p^{(0)})^\top,\\
  G_1&=\diag(p^{(1)})
       -p^{(1)}(p^{(0)})^\top-p^{(0)}(p^{(1)})^\top,\\
  G_2&=\diag(p^{(2)})
       -p^{(2)}(p^{(0)})^\top-p^{(0)}(p^{(2)})^\top
       -p^{(1)}(p^{(1)})^\top.
\end{align*}
Let \(q=(e_i-e_j)/\sqrt2\) and \(w=G_1q\). The coefficient of
\(t^2\) in the simple-eigenvalue expansion is
\begin{equation}\label{eq:eigenvalue-second-order}
  \kappa_2
  =q^\top G_2q+
   w^\top\bigl[(aI_C-G_0)|_{q^\perp}\bigr]^{-1}w.
\end{equation}
The restriction is invertible because \(a\) is simple and strictly
larger than the other eigenvalues of \(G_0\).

Direct substitution gives
\(w_i=w_j=ab/\sqrt2\),
\(w_m=-\sqrt2\,ab_m\) for \(m\notin\{i,j\}\), and
\[
  q^\top G_2q=\frac a2(1-6a).
\]
The solution of \((aI_C-G_0)y=w\) in \(q^\perp\) has coordinates
\[
  y_i=y_j=
  \frac{b+\sum_{m\notin\{i,j\}}b_m^2/(a-b_m)}
       {2\sqrt2\,a},
  \qquad
  y_m=-\frac{b_m}{\sqrt2(a-b_m)}.
\]
Hence
\[
  w^\top y
  =\frac{b^2}{2}
   +\frac12\sum_{m\notin\{i,j\}}
      \frac{b_m^2}{a-b_m}.
\]
Combining these expressions in~\eqref{eq:eigenvalue-second-order},
using \(b=\sum_m b_m\) over the remaining coordinates and
\(2a+b=1\), gives
\[
  \kappa_2
  =\frac a2\left[
    1-6a+\sum_{m\notin\{i,j\}}
    \frac{b_m(b+2b_m)}{a-b_m}\right]
  =\frac a2\psi_{ij}(x_0).
\]
Evenness gives the remainder \(O(t^4)\) in
Equation~\eqref{eq:cclass-expansion}. Finally,
\(\Lambda_p=d_\mu\eig_{\max}(G(p))/(\sqrt2\sigma^2)\);
the prefactor is positive and the logarithm is strictly increasing.
Since \(t\) is a nonzero scalar multiple of \(s\), the sign of
\(\psi_{ij}\) determines the stated strict normal maximum or minimum
when it is nonzero.
\end{proof}

\begin{proof}[Proof of Corollary~\ref{cor:threeclass-transition}]
For \(C=3\), the only remaining posterior coordinate is
\(b=1-2a\). Equation~\eqref{eq:cclass-expansion} becomes
\[
  \psi_{ij}
  =1-6a+\frac{3b^2}{a-b}
  =-\frac{6a^2+3a-2}{3a-1}.
\]
For \(1/3<a<1/2\), the denominator is positive. The numerator
vanishes at \(\ac=(\sqrt{57}-3)/12\), and its sign gives the
two strict regimes through Theorem~\ref{thm:cclass-criterion}.

At \(a=\ac\), the quadratic coefficient vanishes. Let
\(\lambda_+(t)\geq\lambda_-(t)>0\) be the nonzero eigenvalues
of \(G(p(t))\). Because the third eigenvalue is zero,
\[
  \lambda_+(t)+\lambda_-(t)
  =1-\frac{2a^2\cosh(2t)+b^2}
           {(2a\cosh t+b)^2},
  \qquad
  \lambda_+(t)\lambda_-(t)
  =\frac{3a^2b}{(2a\cosh t+b)^3}.
\]
Thus \(\lambda_+(t)\) is half the sum plus the positive square
root of the sum squared minus four times the product. Expanding
these explicit even functions at \(a=\ac\) gives
\[
  \lambda_+(t)
  =\ac+\frac{-87+5\sqrt{57}}{144}t^4+O(t^6).
\]
The quartic coefficient is negative. The spectral reduction and
the monotonicity of the logarithm make the facet a strict
fourth-order normal FTMLE maximum at equality. The directional
statement follows from Proposition~\ref{prop:cclass-direction}
throughout the smooth-facet range.
\end{proof}

\section{Proofs for input-to-logit
results}\label{app:logit-proofs}

\begin{proof}[Proof of Theorem~\ref{thm:mse-logit}]
Conditioning on \(X\) and completing the square gives
\begin{align*}
 R_{\mathrm{MSE}}(\theta)
 &=\E\!\left[
   \operatorname{Var}(\widehat Y\mid X)
   +\bigl(\E[\widehat Y\mid X]
          -\ell_\theta^{\mathrm{MSE}}(X)\bigr)^2\right].
\end{align*}
The first term is independent of \(\theta\). Realizability and
global optimality make the second term vanish almost surely.
The Gaussian mixture density is strictly positive and the relevant
maps are continuous, so
\(\ell_{\theta^\star}^{\mathrm{MSE}}(x)
=\E[\widehat Y\mid X=x]\) pointwise. Since
\(\widehat Y=2Y-1\),
\[
  \E[\widehat Y\mid X=x]
  =2\eta(x)-1
  =2\sigm(\alpha r(x))-1
  =\tanh\!\left(\frac{\alpha r(x)}2\right).
\]
This function has the same zero set as \(r\). Differentiating,
\[
  D\ell_{\theta^\star}^{\mathrm{MSE}}(x)
  =\frac{\alpha}{2}
    \operatorname{sech}^2\!\left(\frac{\alpha r(x)}2\right)n^\top.
\]
It is a nonzero rank-one row Jacobian. Its operator norm is
the displayed positive factor, and its right singular direction is
\(\pm n\). The factor is uniquely maximal as a function of
normal coordinate at \(r=0\), proving the full boundary maximizer
set after the depth-normalized logarithm.
\end{proof}

\begin{proof}[Proof of Corollary~\ref{cor:ce-logit-constant}]
At each \(x\), the conditional CE risk for a scalar logit \(\ell_{\theta}^{\mathrm{CE}}(x)\) is the binary cross-entropy of
\(p=\sigm(\ell_{\theta}^{\mathrm{CE}}(x))\). 
Lemma~\ref{lem:proper-loss} makes its unique minimizing probability \(p=\eta(x)\). 
In the Gaussian model
\(0<\eta(x)<1\), so the sigmoid is invertible there and
\[
  \ell_{\theta^\star}^{\mathrm{CE}}(x)=\log\frac{\eta(x)}{1-\eta(x)}
  =\alpha r(x)
\]
almost surely. Realizability, positive input density, and
continuity upgrade the equality to all \(x\in\R^d\).
Its derivative is the constant row vector \(\alpha n^\top\).
Therefore the leading singular value is \(\alpha\), its right
singular direction is \(\pm n\), and its FTMLE is
\((K+1)^{-1}\log\alpha\) everywhere. Its zero set is
\(\{r=0\}=\B^\star\).
\end{proof}

\section{Proof for input-to-representation
results}\label{app:hidden-ridge}

\begin{proof}[Proof of Proposition~\ref{prop:task-aligned-localization}]
    A simple sufficient pointwise condition for a boundary projection \(\pi(x)\in\B\) is
\begin{equation}\label{eq:pointwise-transfer}
  \Lambda_H(x)<\Lambda_T(\pi(x)),
  \qquad x\notin\B.
\end{equation}
Since \(\Lambda_T(\pi(x))\leq\Lambda_H(\pi(x))\),
condition~\eqref{eq:pointwise-transfer} immediately implies
\(\Lambda_H(x)<\Lambda_H(\pi(x))\).
The following construction supplies quantitative conditions along
all normals of a smooth boundary patch.
\end{proof}

\begin{proof}[Proof of Theorem~\ref{thm:hidden-ridge}]
Fix \(b\in\B_0\), take \(0<|t|<\varepsilon_0\), and set
\(x=b+t n(b)\). The projection and distance assumptions give
\(\pi(x)=b\) and \(d(x)=|t|\). Since \(s\leq[s]_+\),
the bound \(r_T(x)\leq\varepsilon_T\) implies
\[
  \kappa t^2+
  \log\frac{\Lambda_T(x)}{\Lambda_T(b)}
  \leq\varepsilon_T\kappa t^2.
\]
Therefore
\begin{equation}\label{eq:appendix-task-visible-decay}
  \Lambda_T(x)
  \leq\Lambda_T(b)
  \exp(-(1-\varepsilon_T)\kappa t^2).
\end{equation}
For any unit input direction \(u\), the orthogonal
decomposition in~\eqref{eq:hidden-orthogonal-identity} gives
\[
  \norm{Dh_\theta(x)u}_2^2
  \leq\Lambda_T(x)^2+\norm{P_vDh_\theta(x)}_F^2.
\]
Taking the supremum over \(u\), then using
\(r_\perp(x)\leq\varepsilon_\perp\), yields
\[
  \Lambda_H(x)^2
  \leq\Lambda_T(x)^2\bigl(1+\varepsilon_\perp t^2\bigr).
\]
Combining this inequality with
\eqref{eq:appendix-task-visible-decay} and
\(\sqrt{1+s}\leq e^{s/2}\) for \(s\geq0\), we obtain
\[
  \Lambda_H(x)
  \leq\Lambda_T(b)
     \exp\!\left(
       -\left[(1-\varepsilon_T)\kappa
              -\frac{\varepsilon_\perp}{2}\right]t^2
     \right).
\]
The rate condition in the theorem makes the exponent strictly
negative for \(t\neq0\). Since
\(\Lambda_T(b)\leq\Lambda_H(b)\), this proves
\(\Lambda_H(b+t n(b))<\Lambda_H(b)\).
Both quantities are positive because \(\Lambda_T>0\);
the increasing logarithm and positive hidden depth \(K\)
give the stated strict FTMLE inequality. The bounds are
uniform by hypothesis, so the argument applies to every
\(b\in\B_0\) and every admissible nonzero offset.
\end{proof}

\end{document}